\documentclass{article}
\usepackage{authblk}
\usepackage{arxiv}

\usepackage[utf8]{inputenc}
\usepackage[T1]{fontenc}
\usepackage{hyperref}
\usepackage{url}
\usepackage{booktabs}
\usepackage{amsfonts}
\usepackage{nicefrac}
\usepackage{microtype}

\usepackage{amsmath}
\usepackage{amsthm}
\usepackage{cleveref}
\usepackage{lipsum}
\usepackage{graphicx}
\usepackage{natbib}
\usepackage{doi}
\usepackage{float}

\usepackage{times}
\usepackage{subfigure} 
\usepackage{natbib}
\usepackage[nobiblatex]{xurl}

\usepackage{algpseudocode}
\usepackage{algorithm}
\usepackage{algorithmicx}
\usepackage{enumerate}
\usepackage{bbm}
\usepackage{enumitem}
\usepackage{calc}
\usepackage{hyperref}

\usepackage{rotating}
\usepackage{threeparttable}
\usepackage{placeins}
\usepackage{xspace}
\usepackage{comment}

\newtheorem{theorem}{Theorem}[section]
\newtheorem{definition}[theorem]{Definition}
\newtheorem{proposition}[theorem]{Proposition}
\newtheorem{lemma}[theorem]{Lemma}
\newtheorem{assumption}[theorem]{Assumption}

\title{Conformal predictions for longitudinal data}

\date{}

\author[1]{Devesh Batra}
\author[1,$\star$]{Salvatore Mercuri}
\author[1,2]{Raad Khraishi}
\affil[1] {%
	Data Science \& Innovation\\
	NatWest Group\\
	London\\
	United Kingdom$\textsuperscript{\dag}$}
\affil[2] {%
	Institute of Finance and Technology\\
	UCL\\
	London\\
	United Kingdom$\textsuperscript{\ddag}$}

\renewcommand{\headeright}{}
\renewcommand{\undertitle}{}
\renewcommand{\shorttitle}{Conformal predictions for longitudinal data}

\DeclareMathOperator*{\argmin}{\text{argmin}}
\DeclareMathOperator*{\Cov}{\text{Cov}}
\DeclareMathOperator*{\Var}{\text{Var}}
\newcommand{\LPCI}{\text{LPCI}\xspace}
\newcommand{\SPCI}{\text{SPCI}\xspace}
\newcommand{\CQR}{\text{CQR}\xspace}
\newcommand{\TQAB}{\text{TQA-B}\xspace}
\newcommand{\TQAE}{\text{TQA-E}\xspace}
\newcommand{\COVID}{\textsc{covid}\xspace}
\newcommand{\EEG}{\textsc{eeg}\xspace}
\newcommand{\train}{\text{train}\xspace}
\newcommand{\test}{\text{test}\xspace}

\usepackage{array}
\begin{document} 

\maketitle
\begingroup\def\thefootnote{$\star$}\footnotetext{No longer at institute\textsuperscript{1}}\endgroup
\begingroup\def\thefootnote{$\dag$}\footnotetext{\thanks{Correspondence to: devesh.batra@natwest.com}}\endgroup
\begingroup\def\thefootnote{$\ddag$}\footnotetext{\thanks{Correspondence to: raad.khraishi@ucl.ac.uk}}\endgroup

\begin{abstract} 
We introduce Longitudinal Predictive Conformal Inference (\LPCI), a novel distribution-free conformal prediction algorithm for longitudinal data. Current conformal prediction approaches for time series data predominantly focus on the univariate setting, and thus lack cross-sectional coverage when applied individually to each time series in a longitudinal dataset. The current state-of-the-art for longitudinal data relies on creating infinitely-wide prediction intervals to guarantee both cross-sectional and asymptotic longitudinal coverage. The proposed \LPCI method addresses this by ensuring that both longitudinal and cross-sectional coverages are guaranteed without resorting to infinitely wide intervals. In our approach, we model the residual data as a quantile fixed-effects regression problem, constructing prediction intervals with a trained quantile regressor. Our extensive experiments demonstrate that \LPCI achieves valid cross-sectional coverage and outperforms existing benchmarks in terms of longitudinal coverage rates. Theoretically, we establish \LPCI's asymptotic coverage guarantees for both dimensions, with finite-width intervals. The robust performance of \LPCI in generating reliable prediction intervals for longitudinal data underscores its potential for broad applications, including in medicine, finance, and supply chain management.
\end{abstract} 


\section{Introduction}\label{sec:introduction}

The improvement in predictive performance of machine learning models over the last two decades have made them essential components of decision-making pipelines across high-stake domains such as medicine and finance. However, the point estimates yielded by these predictive models are insufficient in these critical application domains, where uncertainty estimates are of particular interest for informed decision-making (see  \cite{harries1999splice, DazGonzlez_2012, Cochran_2015} for uncertainty quantification in these domains). While post-hoc methods such as bootstrapping, jackknife and other ensembling procedures (see \cite{Alaa_2020, Schaar_2020, Xu_2020}) are popularly used for uncertainty estimation of a particular statistic (such as model metrics), they are only able to provide theoretical guarantees under additional assumptions on the underlying model and data distribution. This limitation, however, is addressed by the conformal prediction framework, which provides a principled way to perform model-agnostic and distribution-free uncertainty quantification of complex machine learning models. 

Conformal predictions are a powerful tool for constructing prediction sets or intervals that provide reliable coverage guarantees for the true value. These guarantees are typically based on the assumption of data exchangeability, which is often violated in time series data due to temporal dependencies and non-stationarity. However, accurate uncertainty quantification is crucial for time series data, which is central to applications ranging from medical diagnosis to energy demand and stock market forecasting. Consequently, there has been a growing interest in developing conformal prediction methods that can handle non-exchangeable data. Most of the focus for time series data has been on univariate data, with \cite{Xu_2022} recently developing the Sequential Predictive Conformal Inference (\SPCI) method and showing its asymptotic conditional coverage.

In this work, we study conformal predictions for longitudinal data, which consists of a number of uniquely identified time series. This is often the format of time series data seen in applications in medicine and finance, where we might have a time series for each patient, stock, or customer. In this context, predictions can be made longitudinally, forecasting future values for each time series, or cross-sectionally, predicting values for new, as yet unseen, time series. Consequently, when developing a conformal prediction framework for longitudinal data, two types of coverage need to be addressed: longitudinal coverage, which measures coverage over time for each series, and cross-sectional coverage, which measures coverage over the population for each time-point. Whilst existing methods designed for univariate time series can be applied to this setting (e.g., by modelling each time series individually), they fail to provide cross-sectional coverage guarantees. Moreover, longitudinal coverage guarantees that are asymptotic in the length of the time series, as in the \SPCI method \citep{Xu_2022}, can be suboptimal for longitudinal data with a large population and relatively short time series. Hence, it is imperative to develop conformal prediction frameworks that leverage data from across both longitudinal and cross-sectional axes to construct prediction intervals. Previous work for longitudinal data is limited, with \cite{Kamile_2021} and  \cite{Lin_2022} being notable examples. \cite{Kamile_2021} showed only cross-sectional coverage guarantees, while \cite{Lin_2022} recognised the need to consider both cross-sectional and longitudinal coverage. The \TQAB and \TQAE methods \citep{Lin_2022} improved longitudinal coverage empirically, with \TQAE even providing a theoretical guarantee. However \TQAE suffered from the creation of infinitely-wide intervals in up to 4\% of time-points in experiments \cite[Table 7]{Lin_2022}.

To overcome this, we propose Longitudinal Predictive Conformal Inference (\LPCI), extending the \SPCI framework \citep{Xu_2022} to longitudinal data, which obtained asy mptotic longitudinal coverage of univariate time series without creating infinitely-wide intervals. This involves modelling the conformal scores (in this case the residuals of predictions) as a quantile fixed-effects regression problem, and using this trained quantile regressor to construct prediction intervals. We make further modifications to the underlying approach in order to improve empirical performance, for example by taking exponentially weighted averages in the quantile regressors training data, rather than simple lags as is done in \cite{Xu_2022}. We prove asymptotically that \LPCI achieves theoretical cross-sectional and longitudinal coverage without the creation of infinitely-wide intervals. Our experiments show that \LPCI observes the expected cross-sectional coverage, and longitudinal coverage is greater than or competitive with \TQAB and \TQAE. We also compare to using \SPCI separately on each time series as well as to Conformalized Quantile Regression (\CQR) of \cite{Romano_2019}, with \LPCI outperforming both. Finally, we also show that the widths of the prediction intervals of \LPCI are more adaptive than \CQR, by which we mean that \LPCI produces narrower intervals where the model is more certain and wider intervals where the model is less certain, making it more suitable as a measure of uncertainty for longitudinal data.

\subsection{Related work}
The conformal prediction framework has gained significant traction and has become increasingly popular in the last few years for distribution-free and model-agnostic uncertainty quantification \cite{Angelopoulos_2022}. Originating from work by \cite{Vovk_2005}, \cite{Shafer_2008}, and \cite{Papadopoulos_2008}, the Split conformal method of \cite{Jing_2015} pioneered the type of distribution-free and model-agnostic conformal predictions that are currently prevalent \citep{Gupta_2022, Angelopoulos_2022a, Kivaranovic_2020}. In Split conformal methods, conformal scores are computed on a calibration set and used to measure score quantiles and create intervals. The only assumption on the data is exchangeability, which ensures that the scores calculated on the random calibration set for determining the quantile of the non-conformity scores are transferable to unseen data.

There has been relatively little work on conformal predictions for the cases where exchangeability does not hold, however this is an increasing area of focus. 
Much of the earlier relevant methods in this area addressed the more general problem of distributional shift, which may be used to make conformal predictions in an online learning setting. 
\cite{Tibshirani_2019} considered the covariate shift problem.
\citet{Gibbs_2021} developed a conformal predictive framework to apply to distributional shifts in an online setting, and \citet{Feldman_2022} extended this methodology to apply to arbitrary distributional shifts, making no assumptions on the data distribution.
As noted by several researchers, the general framework of \citet{Gibbs_2021} may be used for time series data, and the specific temporal dependence of such data can be leveraged to fine-tune and improve this approach \citep{Zaffran_2022}.
A common approach to improve the coverage of prediction intervals is to perform leave-one-out (LOO) ensemble predictions to feed into the non-conformal scores, as is done by \citet{Xu_2020} and \cite{Jensen_2022}. 
Whilst these methods may empirically improve longitudinal coverage, it can be difficult to prove asymptotic longitudinal coverage theoretically without additional distributional assumptions. 
The \SPCI algorithm of \cite{Xu_2022} was shown to achieve longitudinal asymptotic coverage for a single \emph{univariate} time series without any additional distributional assumptions.
This was achieved through the introduction of a quantile random forest regressor to predict the required quantile of the non-conformal scores.
Moreover, the \SPCI algorithm improved on the width of the prediction intervals over the former EnbPI method by the same authors \citep{Xu_2020}.
The ideas of \SPCI in using quantile regressors built on previous examples, such as \cite{Romano_2019}, where the authors applied quantile regression to construct intervals for exchangeable data, a method known as Conformalized Quantile Regression (\CQR).

Research for multivariate time series data is limited. 
Nonetheless, longitudinal time series data is considered by \citet{Kamile_2021, Lin_2022}. \citet{Kamile_2021} focused on multi-step forecasting of future time-points for each group. 
They proved a cross-sectional coverage guarantee, but did not introduce any guarantees on the conditional longitudinal coverage of each individual time series.
\cite{Lin_2022} considered the cross-sectional setting, where an estimator is trained on a longitudinal dataset, and developed a conformal prediction algorithm that produces prediction intervals for an \emph{unseen} group. 
\citet{Lin_2022} showed cross-sectional coverage, however longitudinal coverage was only guaranteed with the creation of infinitely-wide intervals at certain time-points.
Our work differs by introducing a conformal prediction framework for longitudinal data that guarantees coverage with finite-width intervals.

\section{Methodology}\label{sec:methodology}

Longitudinal datasets consist of temporal sequences of observations $\{(X_t^{(g)}, Y_t^{(g)})\}_{t=1}^T$ of length $T > 1$, for each group $g\in G = \{1, \dots, |G|\}$, where $Y_t^{(g)}\in\mathbb{R}$ is continuous scalar and $X_t^{(g)}\in\mathbb{R}^d$ consists of features which may contain exogenous features, dates, group identifiers or lags and moving averages of $Y_t^{(g)}$. We are assuming that the dataset is balanced in the sense that there exist observations for each combination of $g$ and $t$ 

We distinguish two types of study for such datasets: \textbf{cross-sectional} and \textbf{longitudinal}. In \textbf{cross-sectional studies} we partition the groups $G = G_{\train}\sqcup G_{\test}$ and train an estimator $\hat{f}$ on $\{(X_t^{(g)}, Y_t^{(g)})\mid 1\leq t\leq T, g\in G_{\train}\}$; point predictions $\hat{Y}_t^{(g)} := \hat{f}(X_t^{(g)})$ are made on the test groups $g\in G_{\test}$ for the same time-points $1\leq t\leq T$ seen in the training set. In \textbf{longitudinal studies} $G = G_{\train} = G_{\test}$ and we train estimator $\hat{f}$ on all groups, $\{(X_t^{(g)}, Y_t^{(g)})\mid 1\leq t\leq T, g\in G\}$; point predictions $\hat{Y}_{T+k}^{(g)} := \hat{f}(X_t^{(g)})$ are made across all groups $g\in G$ at future time-points $k\geq 1$.

In conformal predictions, point predictions are accompanied by prediction intervals $\widehat{C}_{t-1}(X_t^{(g)})$ which should ideally contain the true value $Y_t^{(g)}$ with probability $1-\alpha$ for each group, where $\alpha\in(0, 1)$ is a pre-specified significance level. Such intervals are constructed by measuring (non-)conformity of predictions with the true value. A common scoring function to do so in regression is the residual:
\begin{equation} \label{eq:panel_eps}
	\hat{\epsilon}_t^{(g)} := Y_t^{(g)} - \hat{Y}_t^{(g)}\in\mathbb{R}.
\end{equation}
In traditional conformal predictions, a theoretical guarantee on the coverage level is provided to the effect that the probability of the true values lying in prediction intervals is $1-\alpha$. However, due to the lack of exchangeability in time series data it is generally not possible to obtain such strong guarantees. As discussed in \cite{Lin_2022}, for longitudinal data there are two types of coverage. \emph{Cross-sectional coverage} concerns the level of coverage over the group dimension $g$, whereas \emph{longitudinal coverage} concerns the level of coverage of the temporal dimension $t$. 

\begin{definition}[Asymptotic cross-sectional coverage] \label{def:cscov} We say that the conformal intervals $\widehat{C}_{t-1}(X_t^{(g)})$ have \emph{asymptotic cross-sectional coverage} if, for all $\varepsilon > 0$, there exists $T_0$ such that 
\[
	\Pr(Y_t^{(g)} \in \widehat{C}_{t-1}(X_t^{(g)})) > 1-\alpha - \varepsilon,
\]
for all $t > T_0$ and $g\in G$.
\end{definition}

\begin{definition}[Asymptotic longitudinal coverage] \label{def:longcov} We say that the conformal intervals $\widehat{C}_{t-1}(X_t^{(g)})$ have \emph{asymptotic longitudinal coverage} for a group $g$, if
\[
	\Pr(Y_t^{(g)}\in \widehat{C}_{t-1}(X_t^{(g)})\mid X_t^{(g)}) \to 1 - \alpha\ \text{uniformly as $T\to\infty$.}
\]
\end{definition}

Cross-sectional coverage is marginal over the groups for a fixed given (large enough) time-point; on the other hand, longitudinal coverage is asymptotic in $t$ and conditional over the temporal dimension. As in \cite{Lin_2022}, we make the reasonable assumption that the groups are exchangeable. As a result, we are able to obtain in Theorem~\ref{thm:cscov} an asymptotic cross-sectional coverage guarantee; note however that this is weaker than the corresponding cross-sectional coverage guarantee in \cite[Theorem 3.1]{Lin_2022} which is finite-sample. We also prove asymptotic longitudinal coverage in Theorem~\ref{thm:lpci_cov} by adapting the argument of \cite[Theorem 4.4]{Xu_2022}.

\subsection{\LPCI algorithm}

Longitudinal Predictive Conformal Inference (\texttt{LPCI}) is a conformal prediction method for longitudinal data that extends the \SPCI algorithm of \cite{Xu_2022}. We model the longitudinal data as 
\[
	Y_t^{(g)} = \hat{f}(X_t^{(g)}) + \hat{\epsilon}_t^{(g)},
\]
where $\hat{f}$ is a trained regression model. We assume that $X_t^{(g)}$ contains a group identifier feature alongside other features. This setup is equivalent to a fixed-effects model. The point predictor $\hat{f}$ is being trained on the entire dataset so as to leverage inter-group dependencies and data more efficiently. This distinguishes the method from training \SPCI independently on $G$ univariate time series. In such an approach, each independent \SPCI model has data available only for a single group's time series; this may be quite small and, as we show in Section~\ref{sec:results}, can lead to poor performance of \SPCI. 

Like \SPCI we make use of a past window of residual errors to make quantile predictions and form future prediction intervals. In contrast to \SPCI we create exponentially-weighted means of these residuals. Suppose we are at time $t$ and, for a fixed window size $w\geq 1$, we have observed $w$ values for group $g$ so far. We let
\[
	\mathcal{E}_{t, w}^{(g)} := (\bar{\epsilon}_{t-1}^{(g)}, \dots, \bar{\epsilon}_{t-w}^{(g)})\in \mathbb{R}^w,
\]
denote the lagged window of exponentially-weighted mean residuals for group $g$, where
\[
	\bar{\epsilon}_{k}^{(g)} = k^{-1}\sum_{i = 1}^{k} \gamma^{k-i}\hat{\epsilon}_i^{(g)},
\]
and $\gamma\in[0, 1]$ is the smoothing parameter. Let $F(z\mid \mathcal{E}_{t, w}^{(g)}) = \Pr(\hat{\epsilon_t}^{(g)}\leq z\mid \mathcal{E}_{t, w}^{(g)})$ be the unknown distribution function of the current residual. The true $p$th quantile of the residual is defined by
\begin{equation} \label{eq:trueQ}
	Q_{t, p}^{(g)} := \inf\{e\in\mathbb{R}\mid F(e\mid \mathcal{E}_{t, w}^{(g)})\geq p\},
\end{equation}
where $p\in(0, 1)$; as discussed in \cite{Meinshausen_2006}, for continuous distribution functions, the quantiles are defined by the property $F(Q_{t, p}^{(g)}\mid  \mathcal{E}_{t, w}^{(g)}) = p$ and so we have
\begin{equation}\label{eq:true_coverage}
F(Q_{t, 1-\alpha/2}^{(g)}\mid \mathcal{E}_{t, w}^{(g)}) - F(Q_{t, \alpha/2}^{(g)}\mid \mathcal{E}_{t, w}^{(g)}) = 1-\alpha,
\end{equation}
where $\alpha\in(0, 1)$. The main idea behind both \SPCI and \LPCI is to train a model whose quantile estimates $\widehat{Q}_{t, p}$ converge uniformly to the true quantile as the training data increases. In the limit, intervals constructed using $\widehat{Q}_{t, 1-\alpha/2}$ and $\widehat{Q}_{t, \alpha/2}$ will contain the true value with probability $1-\alpha$ by virtue of Eq.~(\ref{eq:true_coverage}), this is shown in more detail in Section~\ref{sec:theory}. The prediction interval in \LPCI is constructed as
\begin{equation} \label{eq:pi}
	\widehat{C}_{t-1}(X_t^{(g)}) := \left[\hat{Y}_t^{(g)} + \widehat{Q}_{t, \beta}^{(g)}, \hat{Y}_t^{(g)} + \widehat{Q}_{t, 1-\alpha + \beta}^{(g)}\right],
\end{equation}
where $\beta\in[0, \alpha]$ minimises the interval width of $(1-\alpha)$-quantiles:
\begin{equation}\label{eq:beta}
	\beta = \argmin_{p\in[0, \alpha]}\left\{\widehat{Q}_{t, 1-\alpha + p}^{(g)} - \widehat{Q}_{t, p}^{(g)}\right\}.
\end{equation}

To estimate the quantiles we model the past residuals as a fixed-effects model by fitting a quantile random forest regressor on features containing $\mathcal{E}_{t, w}^{(g)}$ and group identifiers, with current residual $\hat{\epsilon}_t$ as the target. At the first date in the testing time-point (either $t = 0$ or $t = T+1$ depending on the setting), there exist $T$ past residuals across groups in $G_{\train}$. We create a training dataset of size $(\tilde{T}|G_{\train}|)\times (w + 1)$, where $\tilde{T} := T - w$, by defining samples and labels, respectively, for $t'=1, \dots, \tilde{T}$, as 
\begin{align}\label{eq:qrf_training_data}
	\begin{split}
	\widetilde{X}_{t}^{(g)} :&= \{(\mathcal{E}_{t + w, w}^{(g)}, g)\}\in\mathbb{R}^{w + 1},  \\
	\widetilde{Y}_{t}^{(g)} :&= \hat{\epsilon}_{t+w}^{(g)}\in\mathbb{R}.
	\end{split}
\end{align}
After each new time-point in the testing data, $|G_{\test}|$ new residuals are observed and the quantile random forest is subsequently retrained on data including the new observations from $G_{\test}$. In the cross-sectional setting, for the first $w$ time-points of the testing period we do not have observed residuals for the test groups. In this case we create dummy residuals of zeros. The full algorithm for \LPCI is given in Algorithm~\ref{alg:LPCI}.

\begin{algorithm}[t]
\caption{Longitudinal Predictive Conformal Inference (\texttt{LPCI}).}
\label{alg:LPCI}
\begin{algorithmic}[1]
\Statex \textbf{Require:} Training data $\{X_t^{(g)}, Y_t^{(g)}\mid g\in G_{\train}, 1\leq t\leq T\}$, where $X_t^{(g)}\in\mathbb{R}^{d}$ and $Y_t^{(g)}\in\mathbb{R}$; regression training algorithm $\mathcal{A}$; significance level $\alpha$; quantile regression training algorithm $\mathcal{Q}$; window size $w$; prediction groups $G_{\test}$ and prediction time-frame $T_{\test}$ ($1\leq t\leq T$ if cross-sectional; $t > T$ if longitudinal).
\Statex \textbf{Output:} Point predictions $\hat{Y}_t^{(g)}\in\mathbb{R}$ and prediction intervals $\hat{C}_{t-1}(X_t^{(g)})\in\mathbb{R}^{2}$, for $g\in G_{\test}$ and all $t\in T_{\test}$.
	\State Train a point predictor $\hat{f} = \mathcal{A}(\{X_t^{(g)}, Y_t^{(g)}\}_{g, t})$.
	\State Use $\hat{f}$ to obtain a panel dataset of residuals $\{\hat{\epsilon}_t^{(g)}\}_{g, t}$ as described in Eq. (\ref{eq:panel_eps}).
	\State Form quantile regression training data $\{\widetilde{X}_{t}^{(g)}, \widetilde{Y}_{t}^{(g)} \mid g\in G_{\train}, 1\leq t\leq \tilde{T}\}$ as described in Eq. (\ref{eq:qrf_training_data}), where $\tilde{T} = T - w, \widetilde{X}_{t}^{(g)}\in\mathbb{R}_{w + 1}, \widetilde{Y}_{t}^{(g)}\in\mathbb{R}$.
    \For {$t' \in T_{\test}$}
	\For {$h\in G_{\test}$}
    		\State Train the quantile regressor $\widehat{\mathcal{Q}} = \mathcal{Q}(\{\widetilde{X}_{t}^{(g)}, \widetilde{Y}_{t}^{(g)}\}_{g, t})$ on the latest residual data.
		\State Obtain the point prediction $\hat{Y}_t^{(h)} = \hat{f}(X_{t'}^{(h)})$.
		\State Form the prediction intervals $\widehat{C}_{t'-1}(X_{t'}^{(h)})$ as described in Eq.~(\ref{eq:pi}) by taking quantile estimates on the latest residual data $\widehat{Q}_{t', p}^{(h)} = \widehat{\mathcal{Q}}(\widetilde{X}_{t'-w}^{(h)})$.
		\State Obtain the latest residuals after observation of true value $\hat{\epsilon}_{t'}^{(h)} = Y_{t'}^{(h)} - \hat{Y}_{t'}^{(h)}$.
		\State Update the training data for the quantile random forest by adding the new samples $\widetilde{X}_{t' - w}^{(h)}$ and $\widetilde{Y}_{t'-w}^{(h)}$. 
    	\EndFor 
	\EndFor
\State\Return $\{\hat{Y}_{t'}^{(h)}, \widehat{C}_{t'-1}(X_{t'}^{(h)})\mid t'\in T_{\test}, h\in G_{\test}\}$.
\end{algorithmic}
\end{algorithm}

\subsection{Theory}
\label{sec:theory}

We show that \texttt{LPCI} obtains asymptotic marginal cross-sectional and conditional longitudinal coverage as the number of training samples of the quantile random forest tends to $\infty$. This limit can be achieved by either taking the number of time-points $T\to\infty$ or the number groups $G\to\infty$ corresponding. However, since the number of groups (in both the cross-sectional and longitudinal case) is generally kept fixed, we only consider the limit in the number of time-points.  All proofs can be found in Appendix~\ref{sec:proofs}.

First we prove asymptotic longitudinal conditional coverage (Definition~\ref{def:longcov}). Our argument follows that of \cite{Xu_2022} with some changes and additional assumptions. The aim is to prove that
\begin{equation}\label{eq:wts}
	\Pr\left(Y_t^{(g)}\in\widehat{C}_{t-1}(X_t^{(g)})\ \big| \ X_t^{(g)}\right)\to 1-\alpha\ \text{uniformly as $T\to\infty$}.
\end{equation}
By definition of the interval, Eq.~(\ref{eq:pi}), we have
\begin{equation} \label{eq:pi_prob}
	\Pr(Y_t^{(g)}\in\widehat{C}_{t-1}(X_t^{(g)})\mid X_t^{(g)}) = F(\widehat{Q}_{t, 1-\alpha+\beta}^{(g)}\mid\mathcal{E}_{t, w}^{(g)}) - F(\widehat{Q}_{t, \beta}^{(g)}\mid\mathcal{E}_{t, w}^{(g)}).
\end{equation}
If we have uniform convergence of the longitudinal quantile random forest to the true quantile: 
\begin{equation}\label{eq:qrf_uniform_conv}
	\sup_{p\in[0, 1]}|\widehat{Q}_{t, p}^{(g)} - Q_{t, p}^{(g)}|\to 0\ \text{as $T\to\infty$},
\end{equation}
then through Eq.~(\ref{eq:true_coverage}), we arrive at the asymptotic guarantee of Eq. (\ref{eq:wts}). So it suffices to prove Eq.~(\ref{eq:qrf_uniform_conv}).

First, consider a \emph{random forest regressor} (non-quantile) trained on the training data $\{(\widetilde{X}_{t}^{(g)}, \widetilde{Y}_{t}^{(g)})\mid g\in G_{\train}, 1\leq t\leq \tilde{T}\}$ containing the residuals. This random forest estimates a conditional cumulative distribution function $\widehat{F}(z\mid x)$ of the residuals. The idea is that if the estimated distribution $\widehat{F}(z\mid x)$ converges uniformly to the true (unknown) distribution of the residuals $F(z\mid x)$, then actually the estimated quantiles of the corresponding \emph{quantile} random forest converge uniformly to the true quantiles --  this is \cite[Proposition 4.2]{Xu_2022} and it holds for our setting as well. So we need only focus on the random forest's distribution function. 

We assume that, for each $t$ and each $g$, the support of the features are contained in a compact space $\text{Supp}(\widetilde{X}_t^{(g)})\subseteq\mathbb{B}\subseteq\mathbb{R}^{d}$. Each parameterised tree $T(\theta)$ in the random forest is trained such that each leaf of the tree is associated with a rectangular subspace of $\mathbb{B}$. These subspaces are disjoint and cover $\mathbb{B}$ so that, for every $x\in\mathbb{B}$ there exists a unique leaf $\ell(x, \theta)$ such that $x\in R_{\ell(x, \theta)}$. Assume that the random forest has parameterised trees $\{T(\theta_k)\}_{k=1}^K$ which have been trained on the training data. Given $x\in\mathbb{B}$, we can give an explicit form of the conditional distribution function of the quantile random forest through weighting past observations as follows:
\begin{align}
	L(x, \theta) :&= \#\{(g, t)\in\{1, \dots, |G_{\train}|\}\times\{1, \dots, \tilde{T}\}\mid \widetilde{X}_t^{(g)}\in R_{\ell(x, \theta)}\}, \label{eq:node_size} \\
	w_t(x, \theta) :&= L(x, \theta)^{-1}\mathbbm{1}(\widetilde{X}_t^{(g)}\in R_{\ell(x, \theta)}), \label{eq:weight} \\
	w_t(x) :&= K^{-1}\sum_{k=1}^K w_t(x, \theta_k). \label{eq:ens_weight}
\end{align}
The value $L(x, \theta)$ counts the total node size of the leaf $\ell(x, \theta)$ across all observed groups and time-steps. The value $w_t(x, \theta)$ weights a relevant leaf by its node-size and finally $w_t(x)$ takes the average across ensemble trees. We have that 
\[
	\sum_{g\in G_{\train}}\sum_{t=1}^{\tilde{T}}w_t(x) = 1.
\]
The cumulative distribution function of the estimates of the \emph{random forest} trained on $\{(\widetilde{X}_t^{(g)}, \widetilde{Y}_t^{(g)})\}_{g, t}$, conditioned on some input $\widetilde{X}_{k}^{(h)}$, is given \citep{Meinshausen_2006} by
\begin{align}\label{eq:qrf_dist_f}
	\widehat{F}(z\mid\widetilde{X}_{k}^{(h)} = x) :&= \sum_{g\in G_{\train}}\sum_{t=1}^{\tilde{T}} w_t(x)\mathbbm{1}(\widetilde{Y}_t^{(g)}\leq z).
\end{align}

\begin{proposition} \label{prop:qrf_conv} Under Assumptions A.1--A.5 of \cite{Xu_2022} and our additional Assumption \ref{assumption:a6} the estimated conditional distribution function $\widehat{F}(z\mid x)$ (Eq.~(\ref{eq:qrf_dist_f})) for the random forest trained on $|G_{\train}| \cdot\tilde{T}$ samples converges uniformly in $z$ for each $x\in\mathbb{B}$:
\[
	\sup_{z\in\mathbb{R}}|\widehat{F}(z\mid x) - F(z\mid x)|\to 0\ \text{in probability as $\tilde{T}\to\infty$}.
\]
\end{proposition}

\begin{theorem}[Asymptotic longitudinal conditional coverage of \texttt{LPCI}]\label{thm:lpci_cov} Under the same assumptions of Proposition \ref{prop:qrf_conv}, we have, for any $\alpha\in(0, 1)$ and any $g\in G$, that
\begin{equation}\label{eq:aymp_cov}
	|\Pr(Y_t^{(g)}\in\widehat{C}_{t-1}(X_t^{(g)})\mid X_t^{(g)}) - (1-\alpha)| \to 0\  \text{in probability as $T\to\infty$}.
\end{equation}
\end{theorem}

\begin{theorem}[Asymptotic cross-sectional marginal coverage of \LPCI]\label{thm:cscov} Let $\varepsilon > 0$ be given. Under the same assumptions of Proposition \ref{prop:qrf_conv} and the additional assumption that the data $\{(X_t^{(g)}, Y_t^{(g)})\}_{g, t}$ is exchangeable in $g$, there exists $T_0$ such that for any $t > T_0$ and any $\alpha\in (0, 1)$, we have 
\[
	\Pr(Y_t^{(h)}\in \widehat{C}_{t-1}(X_t^{(h)})) > 1-\alpha - \varepsilon,
\]
for any $h\in G_{\test}$.
\end{theorem}

\section{Results}\label{sec:results}

In this section, we give results in two experimental settings. In Section~\ref{sec:crossres} we provide results in a cross-sectional experimental setting and in Section~\ref{sec:longitudinal} we provide results in a longitudinal experimental setting. We aim to give empirical results that evidence the theoretical analysis in the previous section. To this end, we verify the following:
\begin{enumerate}
	\item \LPCI observes expected marginal cross-sectional coverage.
	\item \LPCI improves longitudinal coverage rates over literature benchmarks.
	\item \LPCI exhibits high adaptivity of interval width.
\end{enumerate}
With respect to adaptivity of interval widths, the first aim is to construct prediction intervals that are not simply wider than benchmark methods at every data point, as this would artificially improve coverage. Hence, we need to ensure that average interval width is not too large. Additionally, always producing similar interval widths across the board is also generally undesirable -- we want conformal prediction intervals to correctly reflect the uncertainty related to the prediction. Therefore, we usually expect to have a broad range of interval widths widths, see \citet[Section 3.1]{Angelopoulos_2022}. We measure this by measuring the standard deviation of interval widths predicted at each time-point in the test set. Overall, we evaluate $\LPCI$ using the following metrics:
\begin{itemize}
	\item \textbf{Marginal coverage:} the coverage rate averaged across all groups and all time-points. This measures \emph{cross-sectional coverage}.
	\item \textbf{Tail coverage:} the average coverage across the lowest 10\% covered groups. This measures \emph{longitudinal coverage}.
	\item \textbf{Width coefficient of variation (CoV):} measures the standard deviation divided by the mean width, taken across all groups and time-points. This metric penalises widths that have a large mean and narrow distribution, and favours widths with a lower average and broader distribution. This helps to measure \emph{width adaptivity}.
\end{itemize}

\subsection{Cross-sectional Experiments}
\label{sec:crossres}

\textbf{Experimental Setup.} In the cross-sectional setting, we divide our dataset into train and test sets by randomly selecting groups to be entirely contained in either the train or test set, with the temporal dimension being kept the same across both sets. This approach allows us to evaluate the generalisation performance of our model on unseen group categories while maintaining the integrity of the temporal structure in the data.

\textbf{Baselines.} To benchmark \LPCI, we compare against the \TQAB and \TQAE methods from \cite{Lin_2022}. We follow their experimental methodology as mentioned in their paper in order to compare results. We report \TQAB and \TQAE results from \cite{Lin_2022}. We also compare against Conformalized Quantile Regression (\CQR, \cite{Romano_2019}) and a traditional split conformal method (Split) with a tuned random forest regressor as base estimator. We use implementations of \CQR and Split from the MAPIE Python package.\footnote{\href{https://pypi.org/project/MAPIE/}{https://pypi.org/project/MAPIE/}}

\textbf{Datasets.} We perform experiments on two datasets considered by \cite{Lin_2022}: \COVID contains daily cases across 380 local council authorities in the United Kingdom across the month of March 2022; \EEG contains 64 downsampled electroencephalogram readings across one second for 600 patient-trials. Table~\ref{table:datasets} contains experimental details of these datasets.

\begin{table}
\caption{\label{table:datasets} Experimental details for the two tested datasets. $n_x$ and $G_x$ refer to the number of samples and groups in the $x\in\{\train, \test\}$ datasets respectively. $T$ refers to the length of each time series.}
\centering
\begin{tabular}{c | c c | c c | c}
\hline
	& $n_{\train}$ & $n_{\test}$ & $G_{\train}$ & $G_{\test}$ & $T$ \\
\hline
\hline
	$\COVID$ & 9,000 & 2,400 & 300 & 80 & 30\\ \hline
	$\EEG$ & 25,600 & 12,800 & 400 & 200 & 64\\ \hline
\end{tabular}
\end{table}

\textbf{Experimental Details.} Our point predictor is a random forest regression model. This is trained as an 5-fold ensemble model with individual estimators, with groups contained entirely within folds and predictions being made by mean aggregation of the individual learners. We use a single lagged target $Y_{t-1}^{(g)}$ as a feature and label-encoded group identifiers as features. Both the target and lagged target columns are standardised across groups. We tune the hyperparameters of the random forest using a randomised grid search across 5-fold group cross validation splits. 

The significance level is $\alpha = 0.1$.  For training the quantile random forest, we use a window size of 20, yielding a training data consisting of 20 lagged exponentially-weighted averages of past residuals, as described in Section~\ref{sec:methodology}. We also tune the hyperparameters of the quantile random forest using a randomised grid search across 5-fold cross validation splits. As in \cite{Lin_2022}, we report metric values only for the last 20 time-points for each group. A final point to mention in the cross-sectional setting is the need to create dummy residuals for the QRF on the first date, since no historical residuals are available at that point for the test groups; since we are only computing metrics on the final 20 time-points for each group this has limited impact on the results.

\begin{table}[h]
\caption{\label{table:csavgcov} Marginal coverage of \LPCI on the \COVID and \EEG datasets compared to baselines in the cross-sectional setting. Marginal coverages in italics do not satisfy expected coverage rates. Experiments are repeated over five seeds. We report the mean $\pm$ standard deviation across experiments.}
\centering
\begin{tabular}{c | c | c c | c c}
\hline
Dataset & \LPCI & \TQAB & \TQAE & \CQR & Split \\
\hline\hline
\COVID & 0.964 $\pm$ 0.006 & 0.908 $\pm$ 0.015 & 0.917 $\pm$ 0.009 & \textit{0.897 $\pm$ 0.010} & 0.900 $\pm$ 0.018 \\
\hline
\EEG & 0.924 $\pm$ 0.067 & 0.907 $\pm$ 0.015 & 0.906 $\pm$ 0.008 & 0.901 $\pm$ 0.008 & 0.908 $\pm$ 0.004\\
\hline
\end{tabular}
\end{table}

\begin{table}[h]
\caption{\label{table:cstailcov} Tail coverage on the \COVID and \EEG datasets for \LPCI compared to baselines in the cross-sectional setting. Higher values are preferred. The best overall results are \textbf{bolded}; the best results for methods that do not create infinitely-wide intervals are \underline{underlined}. Experiments are repeated over five seeds and we report the mean $\pm$ standard deviation across experiments.}
\centering
\begin{tabular}{c | c | c c | c c}
\hline
Dataset & \LPCI & \TQAB & \TQAE & \CQR & Split \\
\hline\hline
\COVID & \underline{\textbf{0.874 $\pm$ 0.015}}& 0.700 $\pm$ 0.045 & 0.824 $\pm$ 0.013 & 0.779 $\pm$ 0.015 & 0.607 $\pm$ 0.048 \\
\hline
\EEG &\underline{\textbf{ 0.792 $\pm$ 0.007}} & 0.710 $\pm$ 0.022 & \textbf{0.790 $\pm$ 0.012} & 0.711 $\pm$ 0.009 & 0.737 $\pm$ 0.009 \\
\hline
\end{tabular}
\end{table}

\begin{table}[H]
\caption{\label{table:csinvcov} Width CoV for \COVID and \EEG datasets. Higher values are preferred. We do not include the \TQAB or \TQAE methods in this table, since standard deviations are not available from \cite{Lin_2022}, however a comparison of widths alone can be found in Appendix~\ref{appendix:res}. The best results are bolded. Experiments are repeated over five seeds and we report the mean $\pm$ standard deviation across experiments.}
\centering
\begin{tabular}{c  | c | c c}
\hline
Dataset  & \LPCI & \CQR & Split\\
\hline\hline
\COVID  & \textbf{0.536 $\pm$ 0.023}  & 0.391 $\pm$ 0.046 & 0.056 $\pm$ 0.030\\ \hline
\EEG & \textbf{0.275 $\pm$ 0.047} & 0.075 $\pm$ 0.008 & 0.037 $\pm$ 0.014 \\ 
\hline
\end{tabular}
\end{table}

\textbf{Discussion.} Cross-sectional results of \LPCI on \COVID and \EEG datasets for the three metrics described can be found in Tables~\ref{table:csavgcov}, \ref{table:cstailcov}, \ref{table:csinvcov}. We observe \LPCI obtains expected marginal coverage of greater than 0.9 in both cases and gets higher marginal coverage than both \TQAB and \TQAE (Table~\ref{table:csavgcov}). On the \COVID dataset, the same can be said for tail coverage rate (Table~\ref{table:cstailcov}). For the \EEG dataset, \LPCI improves tail coverage over \TQAB and is competitive with \TQAE, however note that this is achieved without the creation of infinitely-wide intervals which occurs in \TQAE. Table~\ref{table:csinvcov} demonstrates that the widths of \LPCI have a greater spread relative to the average width than compared to the \CQR and Split methods. Combined with the improved tail coverage rates from Table~\ref{table:cstailcov}, this shows that \LPCI's intervals are more capable of adapting to harder predictions and suggests that \LPCI exhibits higher width adaptivity than \CQR and Split.

In Figure \ref{fig:ex_intervals} we show some examples of interval widths created by LPCI for the cross-sectional experiment on the \COVID dataset. For LPCI, we saw longitudinal (tail) coverage of 0.874 (Table \ref{table:cstailcov}). Note how, for Barrow-In-Furness and Blenau Gwent, the first half of the month corresponds to relatively flat true and predicted values. Correspondingly LPCI produces narrow intervals, reflecting low uncertainty. In the second half of March 2022, for both of those councils the true values become more volatile, and greater model errors occur. As we would hope to see, there is a widening of LPCI intervals accompanying this phenomenon.

\begin{figure}[h]
	\centering
	\includegraphics[scale=0.28, trim=7cm 7cm 0 7cm, clip]{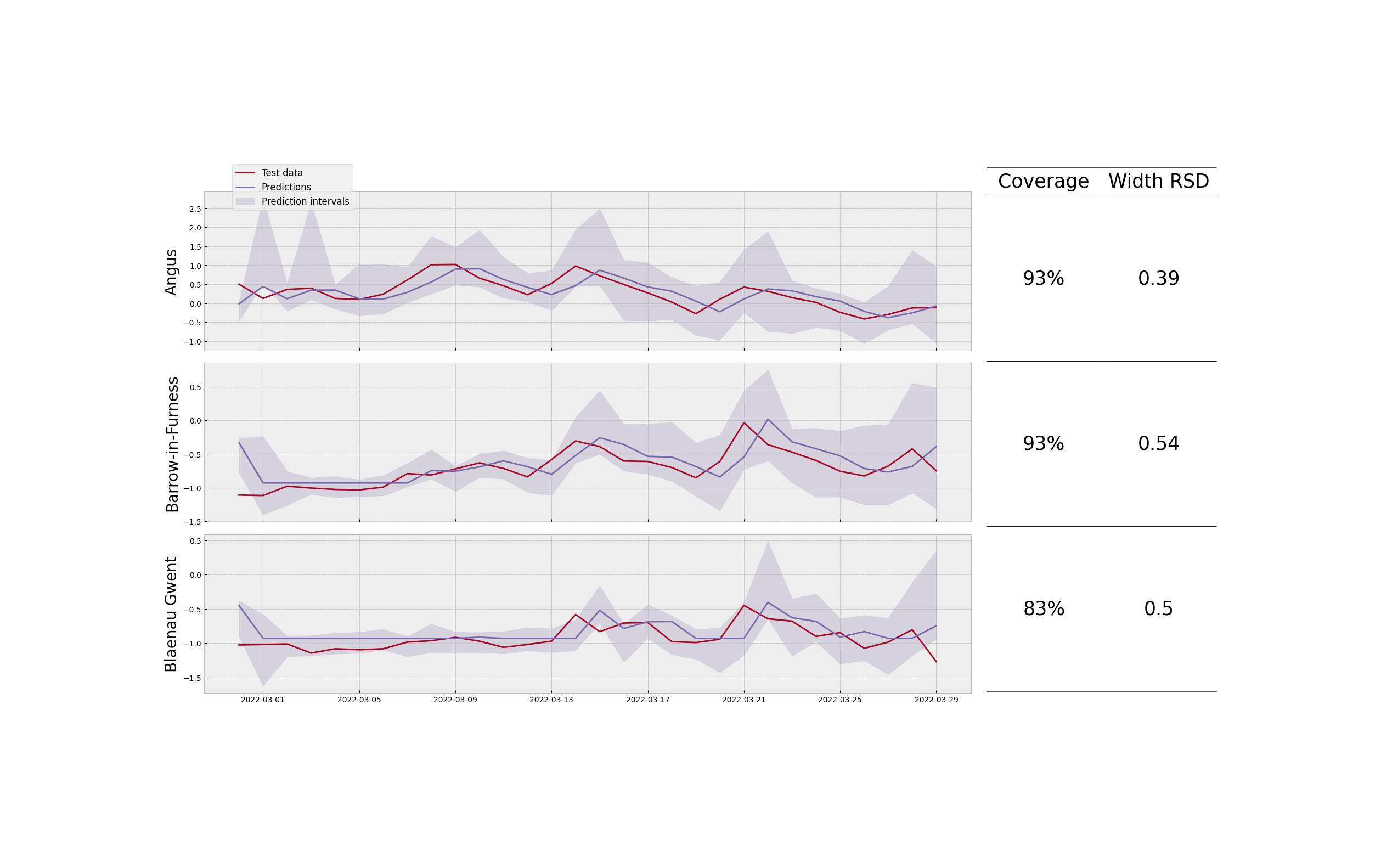}
	\caption{\label{fig:ex_intervals} Examples of \LPCI intervals for 3 councils in the \COVID test set in the cross-sectional experiment. The red line is the observed data, the purple line is the predictions of the underlying model, and the shaded regions are the LPCI intervals. Note that observed values have been normalised.}
\end{figure}

\subsection{Longitudinal Experiments}
\label{sec:longitudinal}

\textbf{Experimental Setup.} In the longitudinal setting, we partitioned our dataset by dividing the data along the time dimension. This allowed us to evaluate the generalisation performance of our model on unseen time steps while ensuring that each group was represented in both the train and test sets.

\textbf{Baselines.} In this case, we are unable to benchmark against \cite{Lin_2022}, since they do not measure results in such a setting. Instead we benchmark against using \SPCI separately on each group as well as \CQR and Split. 

\textbf{Datasets.} We use the \COVID dataset in which we have 380 groups corresponding to local council authorities in the UK. We train on all groups in the month of February 2022 and test on all groups in the month of March 2022. 

\textbf{Experimental Details.} The experiments are almost identical to those in the cross-sectional setting of Section \ref{sec:crossres}, except we have no need to create dummy prediction data for the quantile random forest on the first testing date as historical data exists for all groups.

\begin{table}[h]
\caption{\label{table:longcovid} \LPCI results on the \COVID dataset in the longitudinal experiment setup. Marginal coverages that do not satisfy expected coverage rates are in italics. Up and down arrows indicate whether higher or lower values are preferred for that metric, respectively; best tail coverage is bolded; best width CoV for valid marginal coverages is bolded.  Experiments are repeated over five seeds and we report the mean $\pm$ standard deviation across experiments.}
\centering
\begin{tabular}{c | c | c c c}
\hline
Metric & \LPCI & \SPCI & \CQR & Split \\
\hline
\hline
Marginal coverage & 0.936 $\pm$ 0.004 & \textit{0.657 $\pm$ 0.007} & 0.918 $\pm$ 0.003 & \textit{0.892 $\pm$ 0.001} \\ \hline
Tail coverage $\uparrow$ & \textbf{0.830 $\pm$ 0.005} & 0.539 $\pm$ 0.003 & 0.760 $\pm$ 0.004 & 0.625 $\pm$ 0.001\\  \hline
Width CoV $\uparrow$ & \textbf{0.500 $\pm$ 0.009} & 0.582 $\pm$ 0.007 & 0.304 $\pm$ 0.016 & 0.103 $\pm$ 0.002 \\ \hline
\end{tabular}
\end{table}

\textbf{Discussion.} Table \ref{table:longcovid} contains \LPCI results on the \COVID dataset for the longitudinal setting. Only \LPCI and \CQR obtain expected marginal coverage rates, with \LPCI obtaining slightly better coverage than \CQR. The tail coverage of \LPCI significantly outperforms other methods. 

The width coefficient of variation is higher for \LPCI than it is for \CQR and Split which, as per the discussion in the previous section, indicates a greater width adaptivity for \LPCI. The greater width adaptivity of \LPCI over \CQR is visualised in Figure~\ref{fig:adaptivewidths};  the left plot shows the wider spread of widths for intervals created by \LPCI compared to \CQR, the right plot shows that the group coverage of \LPCI intervals are more concentrated towards higher coverage rates than those of \CQR. This indicates that, for the \COVID dataset, the greater spread of widths for \LPCI is accompanied by improved coverage rates over the worst-covered groups, and therefore that \LPCI is better able to reflect the uncertainty of the model.

The Split method is expected to obtain poor tail coverage, since it does not readily apply to time series data. However, it is interesting to observe the poor performance of \SPCI in this scenario, which does come with longitudinal guarantees for univariate time series. This is the result of two key factors. Firstly we only have 30 training and test points available for each time series and at the final testing date each \SPCI has only seen 60 datapoints in total. This lack of data leads to poor performance on each individual time series. Secondly, \SPCI used in this way is not able to leverage the exchangeability of groups in the dataset. This shows how the \LPCI methodology of modelling a quantile random forest over the entire dataset can help to both model inter-group dependencies and overcome data limitations. 

\begin{figure}[h]
\centering

\includegraphics[scale=0.6, trim=20 0 20 30, clip]{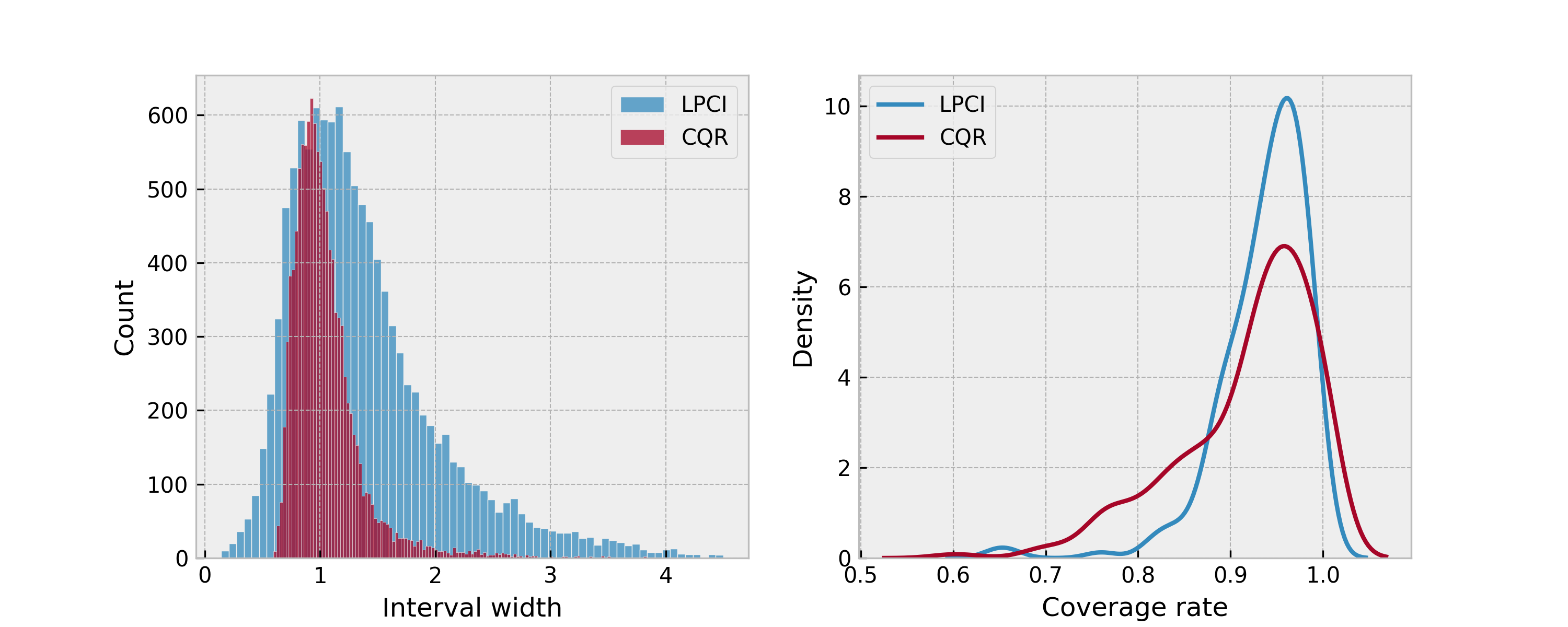}
\caption{\label{fig:adaptivewidths} (left) Distribution of \LPCI widths compared to \CQR widths on the \COVID dataset. We observe a significantly larger spread in both directions for \LPCI than for \CQR. (right) Distribution of coverage rates for each local council authority in the \COVID datasets for \LPCI compared to \CQR.}
\end{figure}

\section{Conclusion}\label{sec:conclusion}

In this paper, we proposed Longitudinal Predictive Conformal Inference (\LPCI) as a distribution-free and model-agnostic uncertainty quantification approach for longitudinal time series data. Our method provides asymptotic theoretical guarantees across both cross-sectional and temporal dimensions. This is done without creating infinitely-wide prediction intervals at any point, which is a first in conformal predictions for longitudinal data. In experiments, we observed that \LPCI was able to improve longitudinal coverage over the current state-of-the-art, whilst also exhibiting a greater adaptivity of widths. Theoretically, however, our framework lacks finite-sample cross-sectional coverage and is only able to guarantee such coverage asymptotically; this could be fixed by combining some ideas from the approach of \cite{Lin_2022}. Moreover, we may be able to leverage more predictive power by looking at alternative quantile regression models, which may help to improve empirical longitudinal coverage even further.

\section*{Acknowledgments}
The authors would like to thank Greig Cowan and Graham Smith of NatWest Group’s Data Science \& Innovation team for the time and support needed to develop this research paper. We would also like to express our sincere gratitude to Greig Cowan for his insightful feedback and constructive suggestions that improved the quality of this paper.

\bibliography{conformal_library}

\begin{thebibliography}{24}
\providecommand{\natexlab}[1]{#1}
\providecommand{\url}[1]{\texttt{#1}}
\expandafter\ifx\csname urlstyle\endcsname\relax
  \providecommand{\doi}[1]{doi: #1}\else
  \providecommand{\doi}{doi: \begingroup \urlstyle{rm}\Url}\fi

\bibitem[Alaa et~al.(2020)Alaa, Gurdasani, Harris, Rashbass, and
  Schaar]{Alaa_2020}
Alaa, Ahmed, Gurdasani, Deepti, Harris, Adrian, Rashbass, Jem, and Schaar,
  Mihaela.
\newblock Machine learning to guide the use of adjuvant therapies for breast
  cancer, 08 2020.

\bibitem[Angelopoulos et~al.(2022)Angelopoulos, Bates, Malik, and
  Jordan]{Angelopoulos_2022a}
Angelopoulos, Anastasios, Bates, Stephen, Malik, Jitendra, and Jordan,
  Michael~I.
\newblock Uncertainty sets for image classifiers using conformal prediction,
  2022.

\bibitem[Angelopoulos \& Bates(2022)Angelopoulos and Bates]{Angelopoulos_2022}
Angelopoulos, Anastasios~N. and Bates, Stephen.
\newblock A gentle introduction to conformal prediction and distribution-free
  uncertainty quantification, 2022.

\bibitem[D{\'i}az-Gonz{\'a}lez et~al.(2012)D{\'i}az-Gonz{\'a}lez, Sumper,
  Gomis‐Bellmunt, and Villaf{\'a}fila-Robles]{DazGonzlez_2012}
D{\'i}az-Gonz{\'a}lez, Francisco, Sumper, Andreas, Gomis‐Bellmunt, Oriol, and
  Villaf{\'a}fila-Robles, Roberto.
\newblock A review of energy storage technologies for wind power applications.
\newblock \emph{Renewable \& Sustainable Energy Reviews}, 16:\penalty0
  2154--2171, 2012.

\bibitem[Feldman et~al.(2022)Feldman, Ringel, Bates, and Romano]{Feldman_2022}
Feldman, Shai, Ringel, Liran, Bates, Stephen, and Romano, Yaniv.
\newblock Risk control for online learning models, 2022.
\newblock URL \url{https://arxiv.org/abs/2205.09095}.

\bibitem[Gibbs \& Candès(2021)Gibbs and Candès]{Gibbs_2021}
Gibbs, Isaac and Candès, Emmanuel.
\newblock Adaptive conformal inference under distribution shift, 2021.
\newblock URL \url{https://arxiv.org/abs/2106.00170}.

\bibitem[Gupta et~al.(2022)Gupta, Kuchibhotla, and Ramdas]{Gupta_2022}
Gupta, Chirag, Kuchibhotla, Arun~K., and Ramdas, Aaditya.
\newblock Nested conformal prediction and quantile out-of-bag ensemble methods.
\newblock \emph{Pattern Recognition}, 127:\penalty0 108496, 2022.
\newblock ISSN 0031-3203.
\newblock \doi{https://doi.org/10.1016/j.patcog.2021.108496}.
\newblock URL
  \url{https://www.sciencedirect.com/science/article/pii/S0031320321006725}.

\bibitem[Harries et~al.(1999)Harries, Wales, et~al.]{harries1999splice}
Harries, Michael, Wales, New~South, et~al.
\newblock Splice-2 comparative evaluation: Electricity pricing.
\newblock 1999.

\bibitem[Jensen et~al.(2022)Jensen, Bianchi, and Anfinsen]{Jensen_2022}
Jensen, Vilde, Bianchi, Filippo~Maria, and Anfinsen, Stian~Normann.
\newblock Ensemble conformalized quantile regression for probabilistic time
  series forecasting.
\newblock \emph{{IEEE} Transactions on Neural Networks and Learning Systems},
  pp.\  1--12, 2022.
\newblock \doi{10.1109/tnnls.2022.3217694}.
\newblock URL \url{https://doi.org/10.1109%2Ftnnls.2022.3217694}.

\bibitem[Kivaranovic et~al.(2020)Kivaranovic, Johnson, and
  Leeb]{Kivaranovic_2020}
Kivaranovic, Danijel, Johnson, Kory~D., and Leeb, Hannes.
\newblock Adaptive, distribution-free prediction intervals for deep networks,
  2020.

\bibitem[Lei et~al.(2015)Lei, Rinaldo, and Wasserman]{Jing_2015}
Lei, Jing, Rinaldo, Alessandro, and Wasserman, Larry.
\newblock A conformal prediction approach to explore functional data.
\newblock \emph{Annals of Mathematics and Artificial Intelligence}, 74\penalty0
  (1–2):\penalty0 29–43, jun 2015.
\newblock ISSN 1012-2443.
\newblock \doi{10.1007/s10472-013-9366-6}.
\newblock URL \url{https://doi.org/10.1007/s10472-013-9366-6}.

\bibitem[Lin et~al.(2022)Lin, Trivedi, and Sun]{Lin_2022}
Lin, Zhen, Trivedi, Shubhendu, and Sun, Jimeng.
\newblock Conformal prediction with temporal quantile adjustments, 2022.
\newblock URL \url{https://arxiv.org/abs/2205.09940}.

\bibitem[Mears et~al.(2015)Mears, Cochran, and Lindsey]{Cochran_2015}
Mears, Daniel, Cochran, Joshua, and Lindsey, Andrea.
\newblock Offending and racial and ethnic disparities in criminal justice: A
  conceptual framework for guiding theory and research and informing policy.
\newblock \emph{Journal of Contemporary Criminal Justice}, 32, 10 2015.
\newblock \doi{10.1177/1043986215607252}.

\bibitem[Meinshausen(2006)]{Meinshausen_2006}
Meinshausen, Nicolai.
\newblock Quantile regression forests.
\newblock \emph{J. Mach. Learn. Res.}, 7:\penalty0 983–999, dec 2006.
\newblock ISSN 1532-4435.

\bibitem[Papadopoulos(2008)]{Papadopoulos_2008}
Papadopoulos, Harris.
\newblock \emph{Inductive Conformal Prediction: Theory and Application to
  Neural Networks}.
\newblock 08 2008.
\newblock ISBN 978-953-7619-03-9.
\newblock \doi{10.5772/6078}.

\bibitem[Romano et~al.(2019)Romano, Patterson, and Candès]{Romano_2019}
Romano, Yaniv, Patterson, Evan, and Candès, Emmanuel~J.
\newblock Conformalized quantile regression, 2019.

\bibitem[Shafer \& Vovk(2008)Shafer and Vovk]{Shafer_2008}
Shafer, Glenn and Vovk, Vladimir.
\newblock A tutorial on conformal prediction.
\newblock \emph{Journal of Machine Learning Research}, 9\penalty0
  (12):\penalty0 371--421, 2008.
\newblock URL \url{http://jmlr.org/papers/v9/shafer08a.html}.

\bibitem[Stankeviciute et~al.(2021)Stankeviciute, M.~Alaa, and van~der
  Schaar]{Kamile_2021}
Stankeviciute, Kamile, M.~Alaa, Ahmed, and van~der Schaar, Mihaela.
\newblock Conformal time-series forecasting.
\newblock In Ranzato, M., Beygelzimer, A., Dauphin, Y., Liang, P.S., and
  Vaughan, J.~Wortman (eds.), \emph{Advances in Neural Information Processing
  Systems}, volume~34, pp.\  6216--6228. Curran Associates, Inc., 2021.
\newblock URL
  \url{https://proceedings.neurips.cc/paper/2021/file/312f1ba2a72318edaaa995a67835fad5-Paper.pdf}.

\bibitem[Tibshirani et~al.(2020)Tibshirani, Barber, Candes, and
  Ramdas]{Tibshirani_2019}
Tibshirani, Ryan~J., Barber, Rina~Foygel, Candes, Emmanuel~J., and Ramdas,
  Aaditya.
\newblock Conformal prediction under covariate shift, 2020.

\bibitem[van~der Schaar et~al.(2020)van~der Schaar, Alaa, Floto, Gimson,
  Scholtes, Wood, McKinney, Jarrett, Lio’, and Ercole]{Schaar_2020}
van~der Schaar, Mihaela, Alaa, Ahmed~M., Floto, R.~Andres, Gimson, Alexander,
  Scholtes, Stefan, Wood, Angela~M., McKinney, Eoin~F., Jarrett, Daniel,
  Lio’, Pietro, and Ercole, Ari.
\newblock How artificial intelligence and machine learning can help healthcare
  systems respond to covid-19.
\newblock \emph{Machine Learning}, 110:\penalty0 1 -- 14, 2020.

\bibitem[Vovk et~al.(2005)Vovk, Gammerman, and Shafer]{Vovk_2005}
Vovk, Vladimir, Gammerman, Alex, and Shafer, Glenn.
\newblock \emph{Algorithmic Learning in a Random World}.
\newblock 01 2005.
\newblock \doi{10.1007/b106715}.

\bibitem[Xu \& Xie(2020)Xu and Xie]{Xu_2020}
Xu, Chen and Xie, Yao.
\newblock Conformal prediction for time series, 2020.
\newblock URL \url{https://arxiv.org/abs/2010.09107}.

\bibitem[Xu \& Xie(2022)Xu and Xie]{Xu_2022}
Xu, Chen and Xie, Yao.
\newblock Sequential predictive conformal inference for time series, 2022.
\newblock URL \url{https://arxiv.org/abs/2212.03463}.

\bibitem[Zaffran et~al.(2022)Zaffran, Dieuleveut, Féron, Goude, and
  Josse]{Zaffran_2022}
Zaffran, Margaux, Dieuleveut, Aymeric, Féron, Olivier, Goude, Yannig, and
  Josse, Julie.
\newblock Adaptive conformal predictions for time series, 2022.
\newblock URL \url{https://arxiv.org/abs/2202.07282}.

\end{thebibliography}
\bibliographystyle{icml2014}
\appendix
\section{Supplementary Material}
\subsection{Assumptions and proofs}
\label{sec:proofs}
\begin{assumption} \label{assumption:a6} Define $U_t^{(g)} := F(\widetilde{Y}_t^{(g)}\mid \widetilde{X}_t^{(g)})$ as the quantile of the $t$th observation of the $g$th group, condition on the feature $\widetilde{X}_t^{(g)}$ with the property that $U_t\sim U[0, 1]$. If $x\in\text{Supp}(\{X_t^{(g)}\}_{t \geq 1})$, then we assume that there exists a function $\tilde{g}$ such that
\[
	\Cov(\mathbbm{1}(U_t^{(g)}\leq F(z\mid x)), \mathbbm{1}(U_t^{(h)}\leq F(z\mid x)))\leq \tilde{g}(t),
\]
for all $t\geq 1$ and all $g\neq h$. Moreover, $\tilde{g}$ has the following bounded growth
\[
	\lim_{\tilde{T}\to\infty}\left[\int_1^{\tilde{T}}\tilde{g}(u)du\right]/\tilde{T}^2 = 0.
\]
\end{assumption}

Note that Assumption~\ref{assumption:a6} is similar to \cite[Assumption A.1]{Xu_2022}, however here we are ranging over the group axis. It essentially assumes subquadratic growth of the cross-sectional covariance as the time-points tend to $\infty$. Note that if the groups are exchangeable, then it is enough to check this assumption once, since all the covariances of pairwise groups are equal.

\begin{proof}[Proof of Proposition \ref{prop:qrf_conv}]
We follow the proof of \cite[Proposition 4.3]{Xu_2022} and \cite[Theorem 1]{Meinshausen_2006} with minor changes due to the slight difference in the form of the estimated conditional distribution function. Note that \cite[Theorem 1]{Meinshausen_2006} assumes i.i.d observations, which neither we nor \cite{Xu_2022} do. We denote $U_t^{(g)} := F(\widetilde{Y}_t^{(g)}|\widetilde{X}_t^{(g)})$ as the quantile of the $t$th observation of the $g$th group; we have $U_t^{(g)} \sim U[0, 1]$. Following the calculation of \cite{Xu_2022}, we get
\begin{align*}
	|\widehat{F}(z\mid x) - F(z\mid x)|\leq &\overbrace{\left|\sum_{g\in\tilde{G}}\sum_{t=1}^{\tilde{T}} w_t(x)\mathbbm{1}(U_t^{(g)}\leq F(z\mid x)) - F(z\mid x)\right|}^{(a)} +  \\
	& + \sum_{g\in\tilde{G}}\left(\underbrace{\left|\sum_{t=1}^{\tilde{T}}w_t(x)[\mathbbm{1}(U_t^{(g)}\leq F(z\mid \widetilde{X}_t^{(g)})) - \mathbbm{1}(U_t^{(g)}\leq F(z\mid x))]\right|}_{(b)}\right).
\end{align*}
The aim is to show that this converges to zero as $\tilde{T}\to\infty$. Term (b) is the same as it was in \cite{Xu_2022}, where they showed convergence to zero. Hence we need only show that term (a) converges to zero. 

Define $U^{(g)} := \sum_g\sum_t w_t(x)\mathbbm{1}(U_t^{(g)}\leq F(z\mid x))$; since $\sum_g\sum_t w_t(x) = 1$ we have $\mathbb{E}[U^{(g)}] = F(z\mid x)$ and, by Chebyshev's inequality, that 
\begin{align*}
	\Pr\left(\left|\sum_{g\in\tilde{G}}\sum_{t=1}^{\tilde{T}} w_t(x)\mathbbm{1}(U_t^{(g)}\leq F(z\mid x)) - F(z\mid x)\right|\geq \varepsilon\right) \leq \Var(U^{(g)})/\varepsilon^2,
\end{align*}
for any $\varepsilon > 0$. We show that $\Var(U^{(g)})\to 0$ as $\tilde{T}\to\infty$. By properties of variance and covariance on linear combinations we obtain
\begin{align}
\Var(U^{(g)}) = & \sum_{g\in\tilde{G}}\Var\left(\sum_{t=1}^{\tilde{T}}w_t(x)\mathbbm{1}(U_t^{(g)}\leq F(z\mid x))\right)\label{termi} \\
& + \sum_{g\neq h}\sum_{i\neq j}w_i(x)w_j(x)\Cov(\mathbbm{1}(U_i^{(g)}\leq F(z\mid x)), \mathbbm{1}(U_j^{(h)}\leq F(z\mid x))) \label{termii}\\
& + \sum_{g\neq h}\sum_{t=1}^{\tilde{T}}w_t^2(x)\Cov(\mathbbm{1}(U_t^{(g)}\leq F(z\mid x)), \mathbbm{1}(U_t^{(h)}\leq F(z\mid x))). \label{termiii}
\end{align}
The summand of $\sum_g$ in (\ref{termi}) is shown to converge to zero in \cite{Xu_2022}. Likewise, the summand of $\sum_{g\neq h}$ in (\ref{termii}) also converges to zero as shown in \cite{Xu_2022} (this requires \cite[Assumption A.1]{Xu_2022}). Under our additional Assumption \ref{assumption:a6} and exchangeability of the groups, the third term converges to zero as follows:
\begin{align*}
\sum_{t=1}^{\tilde{T}} w_t^2(x) \Cov(\mathbbm{1}(U_t^{(1)}\leq F(z\mid x)), \mathbbm{1}(U_t^{(2)}\leq F(z\mid x)) \leq \int_1^{\tilde{T}} \mathcal{O}\left(\frac{1}{\tilde{T}^2}\right)\tilde{g}(u)du \to 0.
\end{align*}
\end{proof}

\begin{proof}[Proof of Theorem~\ref{thm:lpci_cov}]
As $T\to\infty$, the number of samples in the training data for the quantile random forest tends to $\infty$ as well due to the continual residual updates. In this limit, the conditional distribution function $\widehat{F}(z\mid \mathcal{E}_{t, w}^{(g)})$ of the random forest converges uniformly to the true distribution by Proposition~\ref{prop:qrf_conv}. Hence we get uniform convergence $\widehat{Q}_{t, p} \to Q_{t, p}$ of the estimated quantiles of the quantile random forest to the true quantile by \cite[Proposition 4.3]{Xu_2022}. By \cite[Assumption A.4]{Xu_2022} the points of discontinuity of the distribution function $F$ have measure zero, the proof then follows by the continuous mapping theorem. 
\end{proof}

\begin{proof}[Proof of Theorem~\ref{thm:cscov}] By exchangeability of groups, $F(\widehat{Q}_{t, 1-\alpha+\beta}^{(h)}\mid\mathcal{E}_{t, w}^{(h)}) = F(\widehat{Q}_{t, 1-\alpha+\beta}^{(g)}\mid\mathcal{E}_{t, w}^{(g)})$ for any group $g$; likewise for the other quantile $\widehat{Q}_{\beta}^{(h)}$. Using this and Eq.~(\ref{eq:pi_prob}) we obtain
\begin{align*}
\Pr(\hat{\epsilon}_t^{(h)}\leq \widehat{Q}_{1-\alpha + \beta}^{(h)}) &= \mathbb{E}_h\Pr(Y_t^{(h)}\in C_{t-1}(X_t^{(h)})\mid X_t^{(h)}) \\
	&= \frac{1}{|G|}\sum_{g\in G}\left[ F(\widehat{Q}_{t, 1-\alpha+\beta}^{(g)}\mid\mathcal{E}_{t, w}^{(g)}) - F(\widehat{Q}_{t, \beta}^{(g)}\mid\mathcal{E}_{t, w}^{(g)})\right] \\
	&= F(\widehat{Q}_{t, 1-\alpha+\beta}^{(h)}\mid\mathcal{E}_{t, w}^{(h)}) - F(\widehat{Q}_{t, \beta}^{(h)}\mid\mathcal{E}_{t, w}^{(h)}).
\end{align*}
Now given $\varepsilon > 0$, let $\sigma_t>0$ be such that 
\begin{align*}
	\Pr(Q_{t, 1-\alpha + \beta}^{(h)} - \sigma_t < \hat{\epsilon}_t^{(h)} \leq Q_{t, 1-\alpha + \beta}^{(h)}) &< \varepsilon/2, \\
	\Pr(Q_{t, \beta}^{(h)} < \hat{\epsilon}_t^{(h)} \leq Q_{t, \beta}^{(h)} + \sigma_t) &< \varepsilon/2.
\end{align*}
By Proposition~\ref{prop:qrf_conv} and \cite[Proposition 4.3]{Xu_2022}, there exists $T_0$ such that for all $t > T_0$, we have $|\widehat{Q}_{t, p}^{(h)} - Q_{t, p}^{(h)}| < \sigma_t$. Since the cumulative distribution is non-decreasing, we have $F(Q_{t, p} - \sigma_t) \leq F(\widehat{Q}_{t, p}) \leq F(Q_{t, p} + \sigma_t)$. Hence, for $t > T_0$ we get
\begin{align*}
F(\widehat{Q}_{t, 1-\alpha+\beta}^{(h)}\mid\mathcal{E}_{t, w}^{(h)}) - F(\widehat{Q}_{t, \beta}^{(h)}\mid\mathcal{E}_{t, w}^{(h)}) &\geq F(Q_{t, 1-\alpha+\beta}^{(h)} - \sigma_t\mid\mathcal{E}_{t, w}^{(h)}) - F(Q_{t, \beta}^{(h)} + \sigma_t\mid\mathcal{E}_{t, w}^{(h)}).
\end{align*}
Now we have
\begin{align*}
F(Q_{t, 1-\alpha+\beta}^{(h)} - \sigma_t\mid\mathcal{E}_{t, w}^{(h)}) &= F(Q_{t, 1-\alpha+\beta}^{(h)} \mid\mathcal{E}_{t, w}^{(h)}) - \Pr(Q_{t, 1-\alpha + \beta}^{(h)} - \sigma_t   < \hat{\epsilon}_t^{(h)} \leq Q_{t, 1-\alpha + \beta}^{(h)}) \\ & > 1- \alpha +\beta - \varepsilon/2, \\
F(Q_{t, \beta}^{(h)} + \sigma_t\mid\mathcal{E}_{t, w}^{(h)}) &= F(Q_{t, \beta}^{(h)} \mid\mathcal{E}_{t, w}^{(h)}) + \Pr(Q_{t, \beta}^{(h)} < \hat{\epsilon}_t^{(h)} \leq Q_{t, \beta}^{(h)} + \sigma_t) \\ & < \beta  + \varepsilon/2.
\end{align*}
Therefore
\begin{align*}
F(Q_{t, 1-\alpha+\beta}^{(h)} - \sigma_t\mid\mathcal{E}_{t, w}^{(h)}) - F(Q_{t, \beta}^{(h)} + \sigma_t\mid\mathcal{E}_{t, w}^{(h)}) & > 1-\alpha -\varepsilon.
\end{align*}
\end{proof}
\newpage
\subsection{Additional Results}
\label{appendix:res}

Tables \ref{table:widths} and \ref{table:longwidths} give values for width mean and standard deviation of the various methods considered in Section~\ref{sec:results}.

\begin{table}[h]
\caption{\label{table:widths} The mean and standard deviation of interval widths for \COVID and \EEG datasets. Up and down arrows indicate whether higher or lower values are preferred for that metric, respectively. The best results are bolded. Experiments are repeated over five seeds and we report the mean $\pm$ standard deviation across experiments. *Since \TQAE can create infinitely-wide intervals, these numbers are calculated by replacing infinite widths with 2x the maximum finite width \citep{Lin_2022}.}
\centering
\resizebox{\textwidth}{!}{
\begin{tabular}{c | c | c | c c | c c}
\hline
Dataset & Metric & \LPCI & \TQAB & \TQAE & \CQR & Split\\
\hline\hline
\COVID & Mean $\downarrow$ & 0.980 $\pm$ 0.048  & \textbf{0.755 $\pm$ 0.033} & 1.070* $\pm$ 0.308 & 0.947 $\pm$ 0.069 & 0.930 $\pm$ 0.024\\ \cline{2-7}
 & St. dev. $\uparrow$ & \textbf{0.526 $\pm$ 0.044} & $-$ & $-$ & 0.372 $\pm$ 0.065 & 0.053 $\pm$ 0.029 \\
\hline
\EEG & Mean $\downarrow$ & 1.592 $\pm$ 0.116 & 1.315 $\pm$ 0.039 & 1.585* $\pm$ 0.111 & \textbf{1.395 $\pm$ 0.084} & 1.398 $\pm$ 0.067 \\ \cline{2-7}
& St. dev. $\uparrow$ & \textbf{0.41 $\pm$ 0.046} & $-$ & $-$ & 0.106 $\pm$ 0.014 & 0.053 $\pm$ 0.021\\
\hline
\end{tabular}}
\end{table}

\begin{table}[ht]
\caption{\label{table:longwidths} \LPCI mean and standard deviation of widths on the \COVID dataset in the longitudinal experiment setup. Experiments are repeated over five seeds and we report the mean $\pm$ standard deviation across experiments.}
\centering
\begin{tabular}{c | c | c c c}
\hline
Metric & \LPCI & \SPCI & \CQR & Split \\
\hline
\hline
Mean width $\downarrow$ & 1.379 $\pm$ 0.004 & 0.703 $\pm$ 0.012 & \textbf{1.05 $\pm$ 0.014} & 1.124 $\pm$ 0.003  \\ \hline
St. dev. of width $\uparrow$ & \textbf{0.689 $\pm$ 0.012} & 0.409 $\pm$ 0.005 & 0.30 $\pm$ 0.016 & 0.116 $\pm$ 0.002 \\ \hline
\end{tabular}
\end{table}

\end{document}